\documentclass{article}

\usepackage[preprint]{neurips_2021}

\usepackage[utf8]{inputenc} 
\usepackage[T1]{fontenc}    
\usepackage{hyperref}       
\usepackage{url}            
\usepackage{booktabs}       
\usepackage{amsfonts}       
\usepackage{nicefrac}       
\usepackage{microtype}      
\usepackage{xcolor}         
\usepackage{caption}
\usepackage{colortbl}
\usepackage{graphicx}
\usepackage{subcaption}
\usepackage{macros}

\title{Fair When Trained, Unfair When Deployed:\\ Observable Fairness Measures are Unstable in Performative Prediction Settings}

\author{%
  Alan Mishler \\
  J.P. Morgan AI Research \\
  \texttt{alan.mishler@jpmorgan.com} \\
  \And
  Niccolò Dalmasso \\
  J.P. Morgan AI Research \\
  \texttt{niccolo.dalmasso@jpmorgan.com} \\
}

\begin{document}

\maketitle

\begin{abstract}
Many popular algorithmic fairness measures depend on the joint distribution of predictions, outcomes, and a sensitive feature like race or gender. These measures are sensitive to distribution shift: a predictor which is trained to satisfy one of these fairness definitions may become unfair if the distribution changes. In performative prediction settings, however, predictors are precisely intended to induce distribution shift. For example, in many applications in criminal justice, healthcare, and consumer finance, the purpose of building a predictor is to reduce the rate of adverse outcomes such as recidivism, hospitalization, or default on a loan. We formalize the effect of such predictors as a type of \emph{concept shift}—a particular variety of distribution shift—and show both theoretically and via simulated examples how this causes predictors which are fair when they are trained to become unfair when they are deployed. We further show how many of these issues can be avoided by using fairness definitions that depend on counterfactual rather than observable outcomes.
\end{abstract}

\section{Introduction}
Much of the algorithmic fairness literature is concerned with so-called \emph{observable} or \emph{statistical} fairness criteria. These criteria consider the relationship between the predictor, a sensitive feature, and a strictly observable outcome, such as whether a person recidivates or has an adverse health event. For example, in a binary classification setting with a binary sensitive feature, the criterion of \emph{equalized odds} requires the classifier to have equal true and false positive rates for both groups \citep{hardt_equality_2016}, while \emph{equality of predictive values} requires the classifier to have equal positive and negative predictive values for both groups \citep{mitchell_algorithmic_2021}.

In many settings, however, predictors are designed to inform decisions that affect the very outcomes that are the target of prediction. This is true for example in risk assessment, when the predictor is meant to estimate the risk of an adverse outcome so that a decision maker can intervene in order to preempt that outcome. For example, a judge might choose to detain a defendant pretrial to prevent recidivism, or a doctor might choose to treat a patient to prevent complications. Such \emph{performative} predictors (the term used by \citet{perdomo_performative_2020}) induce distribution shift, which affects predictive performance. In practice, this may be dealt with by periodically retraining the predictor. Recent work has identified conditions under which iteratively retrained performative predictors converge to an equilibrium \citep{perdomo_performative_2020}, while other work has examined how performative predictors, with or without fairness considerations, affect the welfare of groups who are subject to their predictions \citep{ensign_runaway_2018, hu_short-term_2018, hashimoto_fairness_2018, Liu2018, mishler_when_2021, damour_fairness_2020, zhang_how_2020}.

In this paper, we focus on the interaction between performativity and fairness. Specifically, we show that a performative predictor which appears to satisfy an observable fairness criterion when trained may not satisfy it when deployed. This is because performative predictors are designed to change a decision making process, which in turn changes observable outcomes, resulting in \emph{concept shift}. Fairness definitions which depend on observable outcomes may no longer be satisfied when the distribution changes. Although this observation is extremely simple, we have not seen it clearly articulated in the literature. 

Naturally, a predictor which is trained to be fair with respect to one population may be unfair if it is deployed in a different population, but performativity can induce unfairness even when the populations in which the predictor is trained and deployed are the same, i.e. there is no \emph{covariate shift}. Perversely, this effect may be larger the more the predictor affects decision making, even though the point of training a predictor in such a setting is to improve decision making in order to improve outcomes. We argue that counterfactual fairness criteria, which are not sensitive to observable outcomes and hence do not suffer this limitation, are preferable in this type of setting.

The remainder of the paper is organized as follows. We define the problem setting in Section \ref{sec:notation} and discuss related work in Section \ref{sec:background}. We formalize the effects of interest theoretically in Section \ref{sec:theoretical} and with simple simulated examples in Section \ref{sec:examples}. In Section \ref{sec:counterfactual}, we argue that the use of counterfactual outcomes and associated counterfactual fairness criteria avoids these issues. We conclude in Section \ref{sec:conclusion}.

\section{Notation and Problem Setting} \label{sec:notation}
Let $A \in \{0, 1\}$ denote a sensitive feature such as race or sex; $X \in \mathcal{X}$ a set of additional covariates such as medical, criminal, or financial history; $D \in \{0, 1\}$ a binary decision or treatment variable such as whether to hospitalize a patient, detain a defendant, or issue credit; and $Y \in \{0, 1\}$ a binary outcome that depends on $D$, such as patient death, recidivism, or default on a loan. We use ``group 0'' and ``group 1'' to refer to the $A = 0$ and $A = 1$ groups. Note that the issues we identify do not depend on the use of binary variables; our analyses can all be generalized to more complex cases.

We identify potential or \emph{counterfactual} outcomes with the notation $Y^{D=d}$. This refers to the outcome that would be observed if, possibly contrary to fact, the decision variable were set to $d$. For example, if $Y$ represents whether a defendant recidivates, and $D$ represents whether a defendant is released ($D=0$) or detained ($D=1$) pre-trial, then $Y^{D=0}$ represents whether a defendant would recidivate if released, while $Y^{D=1}$ represents whether they would recidivate if detained. (Presumably, we'd have $Y^{D=1} \equiv 0$, insofar as recidivism is not possible under detention.) 

Importantly, we only observe a single outcome for each individual, since each individual is either detained or not. The other potential outcome remains counterfactual. This is the ``fundamental problem of causal inference'' \citep{Holland1986}. More formally, we make the following assumption:
\begin{itemize}
    \item[(A1)] $Y = (1 - D)Y^{D=0} + DY^{D=1}$
\end{itemize}
(A1) says that for each individual, the outcome that is observed is the potential outcome corresponding to the treatment they actually receive, meaning for example that an individual's outcome does not depend on the treatment status of others. See \citet{Holland1986, rubin_causal_2005} for an overview of the potential outcomes framework.

We suppose that a classifier $R: \mathcal{A} \times \mathcal{X} \mapsto \{0, 1\}$ is trained to predict $Y$ from the covariates, and that $R$ is then deployed to support future decisions $D$ which may themselves affect $Y$. We contrast the accuracy and fairness of $R$ with respect to $Y$ at two time points $t$: before it is deployed (``pre''), i.e. at training time; and after it is deployed (``post''), i.e. with respect to some future test data distribution. We use subscripts $t$ to indicate time, so that for example $\Pb_\text{pre}$ and $\E_\text{pre}$ refer to the distribution of the training data and expectations over that distribution, respectively, while $\Pb_\text{post}$ and $\E_\text{post}$ refer to the corresponding quantities in the test data. Quantities without subscripts are assumed not to change over time.

\section{Background and Related Work} \label{sec:background}
\subsection{Fairness criteria}
Many fairness criteria are defined with respect to the joint distribution of $A$, $R$, and $Y$. We consider popular fairness criteria that depend on the following quantities: the group-specific \emph{prediction rates} ($\Pb(R = 1 \mid A)$); \emph{positive predictive values} (PPVs: $\Pb(Y = 1 \mid A, R = 1)$) and \emph{negative predictive values} (NPVs: $\Pb(Y = 0 \mid A, R = 0)$); and the \emph{error rates}, meaning the \emph{false positive rates} (FPRs: $\Pb(R = 1 \mid A, Y = 0)$), and \emph{false negative rates} (FNRs: $\Pb(R = 0 \mid A, Y = 1)$).

The criterion of \emph{demographic parity} requires $R \ind A$, so that the prediction rates are the same for group 0 and group 1. \emph{Sufficiency} or \emph{equality of predictive values} requires that $Y \ind A \mid R$, so that the PPVs are the same for group 0 and group 1, and likewise for the NPVs. Finally, \emph{separation} or \emph{equalized odds} requires that $R \ind A \mid Y$, so that the FPRs and FNRs are the same for the two groups \citep{hardt_equality_2016, barocas-hardt-narayanan, mitchell_algorithmic_2021}. 

Counterfactual versions of these fairness-related quantities may be defined by for example substituting $Y^{D=0}$ for $Y$. This substitution makes sense in risk assessment settings, where the goal is to estimate the risk of an adverse outcome absent some intervention, and where the use of observable fairness criteria in these settings can be misleading \citep{coston_counterfactual_2020, mishler_fairness_2021}.

\subsection{Dataset shift}
\textit{Dataset shift}, or distribution shift, generally refers to a difference in distribution between training data and test data \citep{quinonero2009dataset, moreno2012unifying}. In our setting, dataset shift may occur across the two time points. The two main types of dataset shift studied in the literature are:
\begin{itemize}
    \item \textit{covariate shift}, when the marginal distribution of the features changes in time but the conditional distribution of the response given the features does not:
    $$
    \quad \mathbb{P}_\text{pre} (A, X) \neq \mathbb{P}_\text{post} (A, X), \text{ but } \mathbb{P}_\text{pre} (Y|A, X) = \mathbb{P}_\text{post}(Y|A, X).
    $$
    \item \textit{concept shift} (or concept drift), when the marginal distribution of the features remains the same but the conditional distribution of the response given the features changes:
     $$
    \mathbb{P}_\text{pre} (A, X) = \mathbb{P}_\text{post}(A, X), \text{ but } \mathbb{P}_\text{pre} (Y|A, X) \neq \mathbb{P}_\text{post} (Y|A, X).
    $$
\end{itemize}
Detecting and mitigating covariate shift with respect to predictor performance is an active area of research \citep{shimodaira2000improving, sugiyama2007covariate, gretton2008covariate, tibshirani2020conformal, hu2020distributionfree}, and likewise for concept shift \citep{vorburger2006entropy, webb2018analyzing, vovk2020testing}. In the fairness literature, \cite{singh2021fairnesscovariateshift} and \cite{rezaei2021fairnesscovariateshift} have studied the effect of covariate shift on fair classifiers and how to mitigate it. In our work, by contrast, we show that introducing a predictor into a decision making context can induce concept shift for the response $Y$ from pre- to post-deployment, even when no covariate shift is present. This concept shift can affect both the accuracy and fairness of the predictor. 

\subsection{Performative prediction and risk assessment}
In many cases, the purpose of training a predictor is to improve decision making in order to improve overall outcomes. When a predictor is optimized for observable outcomes in such settings, then it is \emph{performative} \citep{perdomo_performative_2020}: the predictor affects the very outcomes it aims to predict. One common setting where performative prediction occurs is risk assessment, in which the predictor targets an adverse outcome such as recidivism or a negative health event. Previous work has illustrated how optimizing predictors for observable accuracy in risk assessment can worsen rather than improving outcomes \citep{mishler_when_2021}. Here, we analyze fairness rather than accuracy. While previous work has shown that observable fairness criteria can be misleading in performative settings \citep{coston_counterfactual_2020}, we show how performativity causes predictors to fail to satisfy the very criteria they are trained to satisfy once they are introduced into a decision making context.

\cite{perdomo_performative_2020} developed conditions under which an iteratively retrained predictor which targets observable outcomes will converge to an equilibrium. By contrast, we propose that in performative contexts, predictors should target counterfactual outcomes, which under reasonable conditions bypasses the issue of performativity and avoids the need for retraining.

In the next section, we illustrate how a change in the decision process (equivalently the ``treatment propensity'') can induce concept shift, which in turn can change the predictive values and error rates. This in turn can cause $R$ to become (more) unfair with respect to equalized odds or equality of predictive values.

\section{Theoretical analysis} \label{sec:theoretical}
If $\Pb_\text{pre}$ and $\Pb_\text{post}$, the data generating processes at $t=\text{pre}$ and $t=\text{post}$, can differ arbitrarily, then the fairness and accuracy of $R$ can also differ arbitrarily across the two distributions. For example, if R is trained on one population and then deployed in a different population, then $\Pb_\text{pre}(A, X)$, the distribution of the covariates at $t=\text{pre}$, may be completely different than at $t=\text{post}$, which may affect how $R$ performs.

However, changes in accuracy and fairness are still likely to occur even if the two populations are identical and all that changes across the two time points is how decisions $D$ are made. To formalize this, let $U$ be a set of unobserved variables such that the vector $W = (A, X, U)$ is sufficient to deconfound the treatment process from the potential outcomes. That is, for $d \in \{0, 1\}$:
\begin{align*}
    D \ind Y^{D=d} \mid W
\end{align*}
In general, outside the context of randomized experiments, decisions are not marginally independent of potential outcomes, i.e. $D \not\ind Y^{D=d}$. For example, in the recidivism setting, judges aim to detain precisely those defendants who are at higher risk of recidivism were they to be released, meaning that $\Pb(D = 1 \mid Y^{D=0} = 1) > \Pb(D = 1 \mid Y^{D=0} = 0)$. In order for the condition above to hold, $W$ should include all observed and unobserved variables that are relevant to both the decision process and the outcome. Another way of understanding this is that conditional on $W$, the treatment assignment is essentially random.

For the remainder of the paper, we make the following assumption:
\begin{itemize}
    \item[(A2)] No covariate shift: $\mathbb{P}_\text{pre}(A, X, R, U, Y^{D=0}, Y^{D=1}) \equiv \mathbb{P}_\text{post}(A, X, R, U, Y^{D=0}, Y^{D=1})$
\end{itemize}
The inclusion of the potential outcomes $Y^{D=0}$ and $Y^{D=1}$ means that the population does not change either in terms of (un)observed covariates \emph{or} in terms of responsiveness to different treatments. Under assumption (A1), this means that the only way that outcomes can change is if the decision process changes. We make this point to emphasize that instability in observable fairness is intrinsic to this problem setting, even when the predictor is applied on exactly the same population on which it was trained. For convenience, define the following quantities: 
\begin{align*}
    & \pi_t(W) = \Pb_t(D = 1 \mid W) \tag{Treatment propensity at time $t$} \\
    & \mu^{D=0}(W) = \E[Y^{D=0} \mid W] \tag{Outcome regression for $Y^{D=0}$} \\
    & \mu^{D=1}(W) = \E[Y^{D=1} \mid W] \tag{Outcome regression for $Y^{D=1}$}
\end{align*}
The quantities $\mu^{D=0}(W)$ and $\mu^{D=1}(W)$ are not indexed by $t$ because under (A2) they do not change over time. Again, only the treatment propensity $\pi_t$ is allowed to change, reflecting the influence of the predictor $R$ on the decision process once it is deployed. Since $R$ is a deterministic function of $A$ and $X$, and $(A, X) \subseteq W$, we could equivalently write $\pi_\text{post}(W)$ as $\pi_\text{post}(W) = \Pb_\text{post}(D = 1 \mid W, R)$, but we choose the simpler form for consistency with $\pi_\text{pre}$.

The subsequent propositions show how changes in the treatment propensity $\pi_t$ from pre- to post-deployment can give rise to concept shift and changes in fairness. All proofs are in the appendix.
\begin{prop}[Concept shift] \label{prop:concept_shift}
\begin{align}
    &\E_t[Y \mid A, X] = \E\left\{\E[\gamma_t(W) \mid A, X]\right\}, \-\ \text{where} \nonumber \\
    &\gamma_t(W) = (1-\pi_t(W))\mu^{D=0}(W) + \pi_t(W)\mu^{D=1}(W) \label{gamma_t}
\end{align}
\end{prop}
From \eqref{gamma_t}, it is easy to see that if $\pi_\text{post}(W) \not\equiv \pi_\text{pre}(W)$, then $\E_\text{post}[Y \mid A, X]$ will in general not equal $\E_\text{pre}[Y \mid A, X]$: changes in the treatment propensity induce concept shift.

We now turn to the fairness metrics discussed above. In the absence of covariate shift, the prediction rates do not change over time, since they don't involve outcomes. However, concept shift will generally induce a change in the predictive values and error rates.

\begin{prop}[Prediction rates] \label{prop:demo_parity}
Under assumption (A2), $\E_\text{pre}[R \mid A] \equiv \E_\text{post}[R \mid A]$.
\end{prop}
It follows immediately from Proposition \ref{prop:demo_parity} that a predictor that achieves demographic parity at training time also achieves demographic parity post-deployment; that is, concept shift does not affect demographic parity.

\begin{prop}[Predictive values] \label{prop:predictive_values}
Under (A1)-(A2), the PPV and NPV of $R$ for group $A = a$ at time $t$ are given by
\begin{align*}
    \E_t[Y \mid A = a, R = 1] &= \frac{\E[\E_t[Y \mid A, X]\one\{A = a\}R]}{\E[\one\{A=a\}R]} \\
    \E_t[1 - Y \mid A = a, R = 0] &= 1 - \frac{\E[\E_t[Y \mid A, X]\one\{A = a\}(1-R)]}{\E[\one\{A=a\}(1-R)]}
\end{align*}
\end{prop}

\begin{prop}[Error rates] \label{prop:error_rates}
The FPR and FNR of $R$ for group $A = a$ at time $t$ are given by
\begin{align*}
    \E_t[R \mid A = a, Y = 0] &= \frac{\E\left[R\one\{A=a\}\left(1 - \E_t[Y \mid A, X]\right)\right]}{\E\left[\one\{A=a\}\left(1 - \E_t[Y \mid A, X]\right)\right]} \\
    \E_t[1 - R \mid A = a, Y = 1] &= \frac{\E\left[(1-R)\one\{A=a\}\E_t[Y \mid A, X]\right]}{\E\left[\one\{A=a\}\E_t[Y \mid A, X]\right]}
\end{align*}
\end{prop}
Combining Propositions \ref{prop:concept_shift}, \ref{prop:predictive_values}, and \ref{prop:error_rates}, we see that if $\pi_\text{pre} \not\equiv \pi_\text{post}$, then the predictive values and error rates may change from pre- to post-deployment. If they change differentially for the $A = 0$ and $A = 1$ groups, then a predictor which is fair pre-deployment will be unfair (or less fair) post-deployment. This is illustrated in the next section.

\section{Example Setting} \label{sec:examples}
We provide a simple example that illustrates how predictors which are fair when trained can become unfair when deployed.

For the purposes of setting up the problem, let $f(p) = \frac{10p}{1 - p + 10p}$, for $p \in (0, 1)$. That is, $f$ multiplies the odds of $p$ by 10. 

In addition to the binary sensitive feature $A$, decision $D$, and outcome $Y$, suppose we have covariates $X = (X^{(1)}, X^{(2)})$, with $X^{(1)} \in \{0, 1\}$ and $X^{(2)} \in \R$. Figure \ref{fig:DAGs}(a) shows a causal DAG representing the data-generating process that produces the training data. There are no unobserved confounders. The variable $X^{(2)}$ is independent of all the other variables, while the decision and outcome both depend on $X^{(1)}$. The specific parameters of this process at time $t=\text{pre}$ are as follows:
\begin{align*}
    A &\sim \Bern(0.5) \\
    X^{(1)} \mid A &\sim \Bern(0.8 - 0.2A) \\
    X^{(2)} &\sim N(0, 1) \\
    D \mid X^{(1)} &\sim \Bern(0.3 + 0.5X^{(1)}) \\
    Y^{D=0} \mid X^{(1)} &\sim \Bern(0.3 + 0.5X^{(1)}) \\
    Y^{D=1} \mid X^{(1)} &\sim \Bern(f(0.3 + 0.5X^{(1)}))
\end{align*}
\begin{figure}
    \centering
    \begin{subfigure}{0.3\textwidth}
    \includegraphics[width=\linewidth]{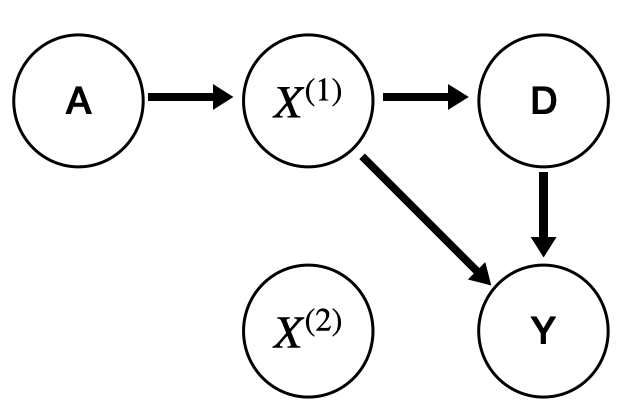}
    \caption{DAG at time $t=\text{pre}$ for the example setting.}
    \end{subfigure}
    \begin{subfigure}{0.3\textwidth}
    \includegraphics[width=\linewidth]{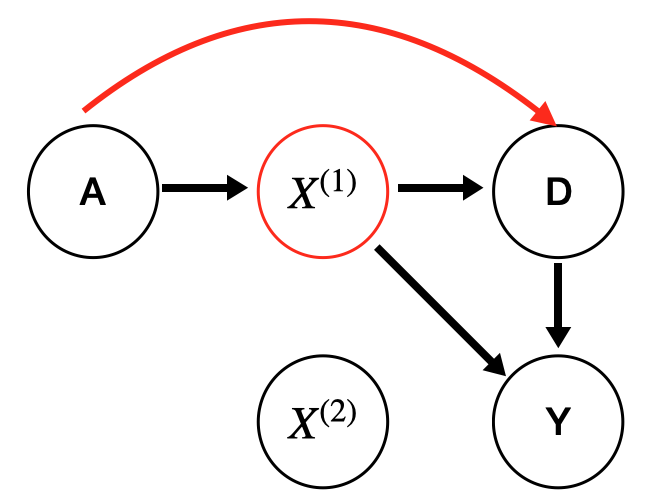}
    \caption{DAG at time $t=\text{post}$ for Predictor 1 in the example setting.}
    \end{subfigure}
    \begin{subfigure}{0.3\textwidth}
    \includegraphics[width=\linewidth]{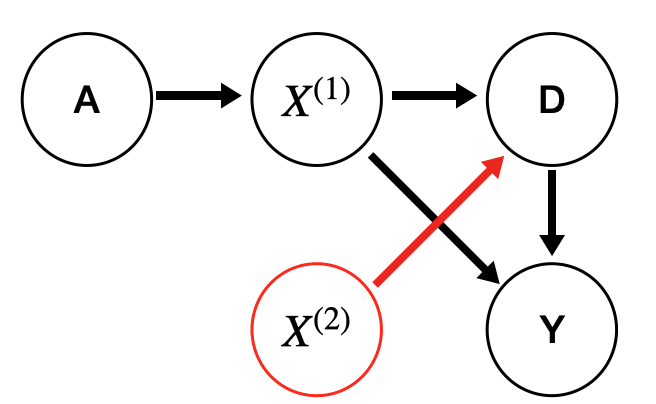}
    \caption{DAG at time $t=\text{post}$ for Predictor 2 in the example setting.}
    \end{subfigure}
    \caption{DAGs for the example setting. Red circles indicate that the predictor $R$ is a function of $X^{(1)}$ (predictor 1) or $X^{(2)}$ (predictor 2). Red arrows indicate causal connections at time $t=\text{post}$ that are not present at time $t=\text{pre}$, due to the influence of $R$.}
    \label{fig:DAGs}
\end{figure}
As one possible interpretation of this setting, suppose we are interested in consumers applying for a loan. Let $A$ be sex, let $X^{(1)}$ indicate whether an applicant has high income, let $D$ indicate whether the loan is approved, and let $Y$ represent whether the applicant becomes a homeowner within a specified time window.

80\% of applicants in group 0 are high income, vs. 60\% in group 1. High income increases the likelihood of receiving a loan ($\Pb(D = 1)$), and there is no difference in this propensity based on sex. Without the loan, low-income applicants have a $30\%$ chance of becoming a homeowner, while high-income applicants have an 80\% chance. With the loan, the odds of becoming a homeowner are multiplied by 10.

Suppose that a predictor $R$ is introduced to help lenders decide whether to issue a loan. In many high-stakes settings, predictive tools cannot legally be used to render automatic decisions, and decision makers have full discretion to utilize information from a predictive tool in a manner they see fit \citep{green2021FlawsPoliciesRequiring}. Hence, decision making processes can in principle change arbitrarily after the introduction of a predictor. We examine the behavior of two possible predictors $R$, coupled with two possible changes in the decision process that result when $R$ is deployed. Since our interest is in illustrating how changes in the decision process can render a ``fair'' predictor unfair in deployment, we do not belabor the mechanism by which $R$ changes the decision process. Each predictor is trained using a training set of size 10,000.

\subsection{Predictor 1}
The first predictor we consider is defined by $R(A, X^{(1)}) = \widehat{\E}_\text{pre}[Y \mid X^{(1)}]$. Since this $R$ is a function of $X^{(1)}$ only, it follows from Figure \ref{fig:DAGs}(a) that at time $t=\text{pre}$, $Y \ind A \mid R$, meaning $R$ satisfies equality of predictive values.

For simplicity, the only quantity that changes over time in this scenario is $\Pb_t(D = 1 \mid A = 1, R = 0)$, the loan approval rate for applicants in group 1 with a negative prediction. Relative to $t=\text{pre}$, the odds of approval for applicants in this group at $t=\text{post}$ is multiplied by a value ranging from 1 to 10,000. This results in an increase in the loan approval probability for this group of between 0 and roughly 0.60. Although this example is simplified for the purposes of illustration, this increase could arise as a form of affirmative action, in which loan officers increase approvals for applicants in the disadvantaged group (the group with lower overall income) who might otherwise not become homeowners.

Figure \ref{fig:predictor1} shows the PPV, NPV, and accuracy at $t=\text{post}$ for groups 0 and 1, as well as the absolute differences between the two groups. When there is no change in approval rates from pre to post (the point 0 on the x-axis), the PPV and NPV remain the same for the two groups. As the change in loan approval rates for group 1 increases, the NPVs for this group decrease, which causes the difference in NPVs between the two groups to increase. That is, $R$ becomes less and less fair at $t=\text{post}$, according to the equality of predictive values criterion. Furthermore, the greater the impact of $R$ on the decision process, the worse the accuracy becomes for group 1. This causes the difference in accuracy between the two groups to increase.

\begin{figure}[ht]
    \centering
    \includegraphics[width=\linewidth]{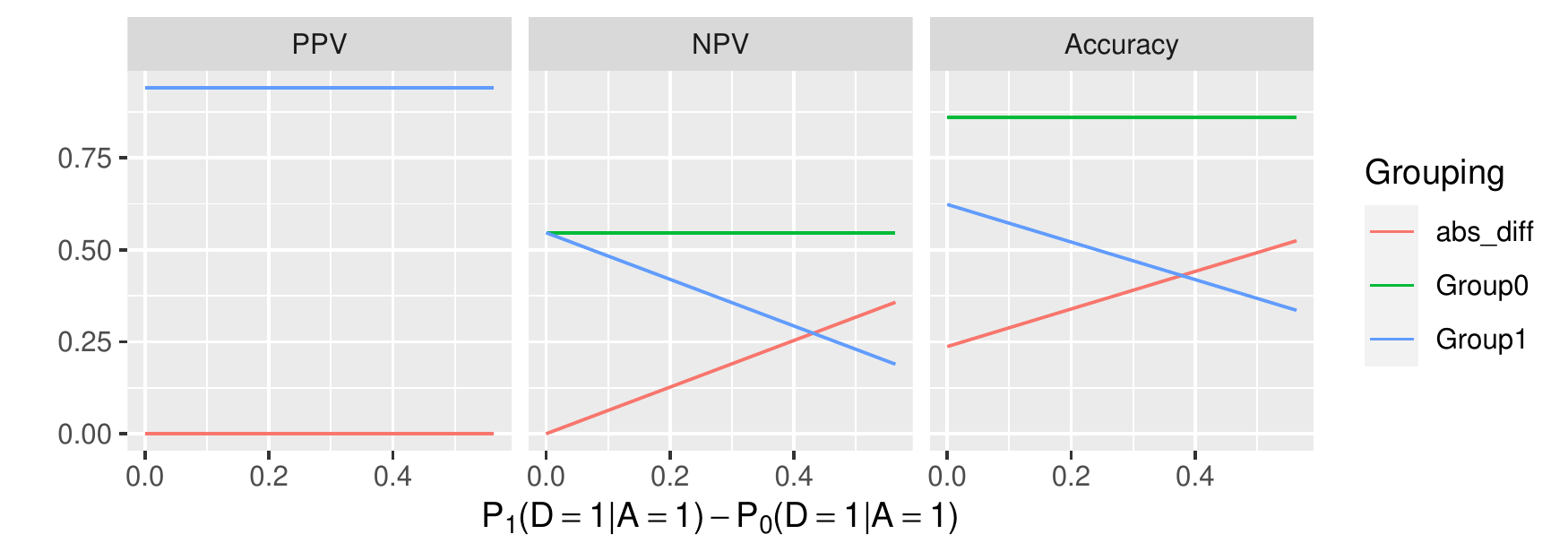}
    \caption{(Predictor 1) Positive predictive values (PPVs), negative predictive values (NPVs), and accuracy for the two groups at time $t=\text{post}$. abs\_diff represents the absolute difference between the two groups in the relevant metric. The x-axis represents the change in loan approval rates for group 1 at $t=\text{post}$ relative to $t=\text{pre}$, with all other components of the decision process held constant. $R$ is fair according to equality of predictive values at $t=\text{pre}$. The greater the impact of $R$ on the decision process, the less fair with respect to equality of predictive values $R$ becomes, and the more inaccurate, at $t=\text{post}$.}
    \label{fig:predictor1}
\end{figure}

\subsection{Predictor 2}
The second predictor we consider is $R(X_2) = \one\{X^{(2)} \geq 0.5\}$. From Figure \ref{fig:DAGs}(c), it follows that at $t=\text{pre}$, $R \ind A \mid Y$, meaning $R$ satisfies equalized odds.

Once again, for simplicity we only vary one component of the decision process: $\Pb_t(D = 1 \mid A = 1, R = 1)$, the loan approval rate for applicants in group 1 with a positive prediction. Relative to $t=\text{pre}$, the odds of approval for applicants in this group at $t=\text{post}$ is multiplied by a value ranging from 1 to 10,000. This results in an increase in the loan approval rate for this group of between 0 and roughly 0.32. This could represent a different type of affirmative action from the previous section, in which loan approvals are increased for the complementary subset of applicants in the disadvantaged group, namely those who are predicted to achieve home ownership.

Figure \ref{fig:predictor2} shows the FPRs, FNR, and accuracy at $t=\text{post}$ for groups 0 and 1, as well as the absolute difference between the two groups. When there is no change in approvals, the FPR and FNR remain the same for the two groups (with some slight differences observed due to sampling error). As the change in loan approval rates for group 1 increases, the FPRs and FNRs for this group decrease, which causes the difference in error rates between the two groups to increase. This means that $R$ becomes less and less fair at $t=\text{post}$, according to the equalized odds criterion. In this case, the accuracy of $R$ improves as the loan approval rate for group 1 increases, although this results in an increasing difference in accuracy between the two groups.

\begin{figure}[ht]
    \centering
    \includegraphics[width=\linewidth]{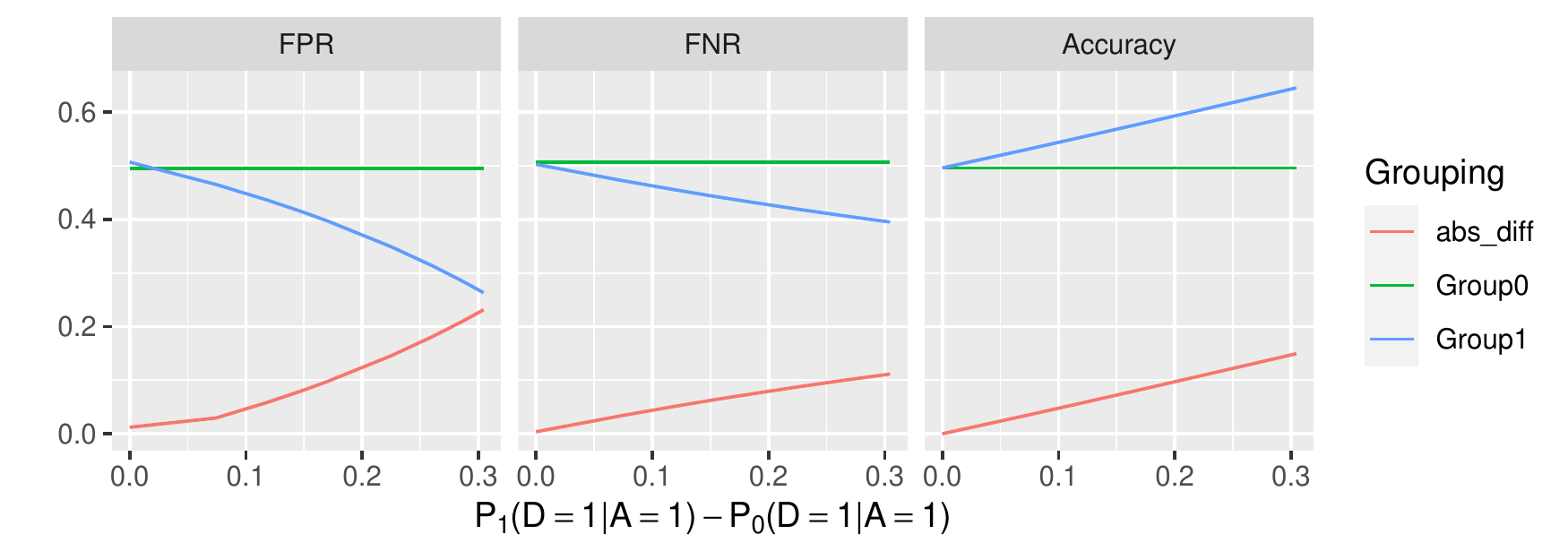}
    \caption{(Predictor 2) False positive rates (FPRs), false negative rates (FNRs), and accuracy for the two groups at $t=\text{post}$. abs\_diff represents the absolute difference between the two groups in the relevant metric. The x-axis represents the change in loan approval rates for group 1 at $t=\text{post}$ relative to $t=\text{pre}$, with all other components of the decision process held constant. $R$ is fair according to equalized odds at $t=\text{pre}$. The greater the impact of $R$ on the decision process, the less fair with respect to equalized odds $R$ becomes at $t=\text{post}$.}
    \label{fig:predictor2}
\end{figure}

\section{Counterfactual accuracy and fairness} \label{sec:counterfactual}
The previous examples illustrate how predictors which satisfy a chosen observable fairness criterion with respect to the data generating process used to train them can fail to satisfy that same criterion when they are deployed. The whole point of introducing a predictor into a decision making setting is to change the decision process in order to improve outcomes. Perversely, the larger the effect of the predictor on decisions, the greater the potential for the fairness and performance of the predictor to differ between training and deployment. These effects occur even when distribution of all other variables, including observed and unobserved covariates and potential outcomes, remains the same.

The use of counterfactual rather than observable outcomes avoids this issue. When a predictor is designed to inform decisions, the outcomes of interest are not the historical observed outcomes under a particular decision process, but rather the outcomes that would occur under available courses of action. In particular, in risk assessment, it is natural to target $Y^{D=0}$, where $D = 0$ represents a baseline course of action such as releasing a defendant or sending a patient home \citep{coston_counterfactual_2020}. In general, $Y^0$ and $Y^1$ represent an individual's responsiveness to different courses of action $D$, but $D$ should not itself affect $Y^0$ or $Y^1$. Under assumption (A1), for example, a predictor which satisfies counterfactual equalized odds ($R \ind A \mid Y^{D=0}$) or counterfactual equality of predictive values ($Y^{D=0} \ind R \mid A$) at $t=\text{pre}$ will also satisfy it at $t=\text{post}$, regardless of changes in the treatment process.

\section{Conclusion} \label{sec:conclusion}
We showed theoretically and in simulated examples that performative prediction settings can induce concept shift, which in turn can affect error rates (false positive and false negative rates), positive and negative predictive values, and accuracy, with respect to observable outcomes. These changes can cause a predictor which satisfies an observable fairness criterion at training time to fail to satisfy this criterion when it is deployed, and they can also cause a predictor to become less accurate in deployment. These phenomena can occur even when the population is identical across the two time points, i.e. when there is no covariate shift, simply as a result of changes in the decision making process. By contrast, concept shift alone has no impact on counterfactual fairness criteria such as counterfactual sufficiency and counterfactual equalized odds.

These results bring into question the value of observable fairness measures in performative contexts, and they add to previous results that suggest that counterfactual outcomes are more natural targets in such settings.

\paragraph{Disclaimer}
{\scriptsize This paper was prepared for informational purposes by the Artificial Intelligence Research group of JPMorgan Chase \& Co. and its affiliates (``JP Morgan''), and is not a product of the Research Department of JP Morgan. JP Morgan makes no representation and warranty whatsoever and disclaims all liability, for the completeness, accuracy or reliability of the information contained herein. This document is not intended as investment research or investment advice, or a recommendation, offer or solicitation for the purchase or sale of any security, financial instrument, financial product or service, or to be used in any way for evaluating the merits of participating in any transaction, and shall not constitute a solicitation under any jurisdiction or to any person, if such solicitation under such jurisdiction or to such person would be unlawful.}

\bibliographystyle{plainnat}
\bibliography{references.bib}

\newpage

\appendix

\section{Appendix: Proofs}
Recall that $R$ is a deterministic function of $(A, X)$, that $(A, X) \subseteq W$, and that $D \ind Y^{D=d} \mid W$ for $d \in \{0, 1\}$.

\begin{proof}[Proof of Proposition \ref{prop:concept_shift} (Concept shift)]We have
\begin{align*}
	\E_t[Y \mid W] &= \E_t[(1 - D)Y^{D=0} + DY^{D=1} \mid W] \tag{by assumption A1} \\
	               &= (1 - \pi_t(W))\mu^{D=0}(W) + \pi_t(W)\mu^{D=1}(W) \tag{since $D \ind Y^{D=d} \mid W$} \\
	               &= \gamma_t(W)
\end{align*}
Now we have
\begin{align*}
	\E_t[Y \mid A, X] &= \E[\E[Y \mid W] \mid A, X] \tag{by iterated expectation} \\
	                  &= \E[\gamma_t(W) \mid A, X]
\end{align*}
\end{proof}

 \begin{proof}[Proof of Proposition \ref{prop:demo_parity} (Prediction rates)]
We have
\begin{align*}
	\E_t[R \mid A] &= \E_t[\E_t[R \mid A, X] \mid A] \\
	               &= \int_\mathcal{X} R d\Pb_t(X \mid A)
\end{align*}
Under assumption (A2), $\Pb_\text{pre}(X \mid A) \equiv \Pb_\text{post}(X \mid A)$, so $\E_t[R \mid A]$ does not change from $t=\text{pre}$ to $t=\text{post}$.
It follows that if $\E_\text{pre}[R \mid A = 0] = \E_\text{pre}[R \mid A = 1]$, then $\E_\text{post}[R \mid A = 0] = \E_\text{post}[R \mid A = 1]$: if the classifier achieves demographic parity at $t=\text{pre}$, then it achieves demographic parity at $t=\text{post}$.
\end{proof}

\begin{proof}[Proof of Proposition~\ref{prop:predictive_values} (Predictive values)]
Following the same logic used in the previous proof, the positive predictive values for group $A = a$ at time $t$ can be expressed as:
\begin{align}                   
	E_t[Y \mid A=a, R=1] &= \frac{\E_t[Y\one\{A=a\}R]}{\E_t[\one\{A = a\}R]} \\
	               &= \frac{\E[\E_t[Y \mid A, X]\one\{A=a\}R]}{\E[\one\{A = a\}R]}
\end{align}
Again, the reasoning for the negative predictive values is analogous, with $(1-R)$ and $(1-Y)$ replacing $R$ and $Y$.
\end{proof}

\begin{proof}[Proof of Proposition \ref{prop:error_rates} (Error rates)]
Starting with the false positive rates, note that we have
\begin{align}
	\E_t[R\one\{A=a\}(1-Y)] &= \E_t[\E_t[R\one\{A=a\}(1-Y) \mid A, Y]] \\
	                        &= \E_t[R \mid A=a, Y=0]\Pb_t(A = a, Y = 0) + 0
\end{align}
We can therefore express the FPR for group $A = a$ at time $t$ as
\begin{align}                    
	E_t[R \mid A=a, Y=0] &= \frac{\E_t[R\one\{A=a\}(1-Y)]}{\Pb(A = a, Y = 0)} \\
	               &= \frac{\E_t[R\one\{A=a\}(1-Y)]}{\E_t[\one\{A = a\}(1-Y)]} \\
	               &= \frac{\E[\E_t[R\one\{A=a\}(1-Y)\mid A, X]]}{\E[\E_t[\one\{A = a\}(1-Y) \mid A, X]]} \\
	               &= \frac{\E[R\one\{A=a\}(1-\E_t[Y \mid A, X])]}{\E[\one\{A = a\}(1-\E_t[Y \mid A, X])]} \label{eq:FPR}
\end{align}
where the third line uses iterated expectation. Note that the outer expectations in the third and fourth lines are not indexed by time, because $\Pb_\text{pre}(A, X) \equiv \Pb_\text{post}(A, X)$ by assumption (A1). The reasoning for the false negative rates is analogous, with $(1-R)$ replacing $R$ and $Y$ replacing $(1-Y)$.
\end{proof}

\end{document}